\documentclass[runningheads]{llncs}

\usepackage{eccv}

\usepackage{eccvabbrv}

\usepackage{graphicx}
\usepackage{booktabs}

\usepackage[accsupp]{axessibility}  

\usepackage{hyperref}

\hypersetup{colorlinks,linkcolor={blue},citecolor={purple},urlcolor={red}} 

\usepackage{orcidlink}

\usepackage{amsmath, amssymb}
\usepackage{geometry}

\usepackage{xspace}
\usepackage{wrapfig}
\usepackage{algorithm}
\usepackage{algorithmic}

\usepackage{multicol}
\usepackage{multirow}
\usepackage{makecell}
\usepackage{tabularx}
\usepackage{adjustbox}
\usepackage{enumitem}

\usepackage{pifont}
\usepackage{color}
\usepackage{colortbl}

\begin{document}

\title{Your Data Manifold is Secretly a Reward Model:\\ Shell-LCC for Text-to-Video Generation}

\titlerunning{Shell-LCC for Text-to-Video Generation}

\author{Shihao Zhang\inst{1}
\and
Yunzhi Li\inst{1}\and
Yuguang Yan\inst{2} \and
Junzhe Zhang\inst{1} \and
Wei Zhao\inst{1} \and
Bohan Wang\inst{1} \and
Hanwang Zhang\inst{1}}


\authorrunning{S.~Zhang et al.}

\institute{Huawei Central Research Institute \and 
Guangdong University of Technology \\
\url{https://needylove.github.io/Shell-LCC/}}

\maketitle

\begin{abstract}
Recent text-to-video (T2V) diffusion models rely heavily on auxiliary reward signals (e.g., via reward models or DPO) to align generated content with human aesthetics and improve realism. These signals, however, incur substantial computational overhead, require costly human annotations, and often yield limited improvement in fine-grained local details. In this paper, we argue that \textbf{your data manifold is secretly a reward model}. By explicitly modeling the manifold structure of high-quality Supervised Fine-Tuning (SFT) data and encouraging video latents to lie on this manifold,
we derive dense, differentiable, and nearly cost-free reward signals that significantly improve video quality, particularly in mitigating low-level distortions.
Our modeling builds upon Local Coordinate Coding (LCC), which captures the `skeleton' of the manifold. However, directly applying LCC suffers from mean regression, pulling latents toward the geometric mean and losing high-frequency details. We therefore extend it to Shell Local Coordinate Coding (Shell-LCC), which models the manifold `surface' as an isotropic shell to align with the true high-density region.
Experiments demonstrate that our approach improves realism, enhances high-frequency details, reduces over-smoothing artifacts, and alleviates motion blur.

\end{abstract}


\section{Introduction}
\label{sec:introduction}

Recent advances in Text-to-Video (T2V) generation \cite{wan2025, wang2025discretevisualtokensautoregression} have been driven by large-scale diffusion models \cite{brooks2024video, blattmann2023stable, bar2024lumiere} trained on paired text-video datasets. While these models demonstrate impressive generative capabilities, they still suffer from notable limitations in realism and aesthetics \cite{huang2024vbench, liu2024evalcrafter, bansal2024videophy}, such as blurred objects, physically implausible movements, or an `over-smoothed' texture that betrays their artificial origin. To mitigate these issues,
recent works \cite{liu2025improving, fan2023dpok, lee2023aligning} have explored Reinforcement Learning from Human Feedback (RLHF)--style fine-tuning strategies, including Direct Preference Optimization (DPO) \cite{rafailov2023direct, wallace2024diffusion, liu2025videodpo} and Group Relative Policy Optimization \cite{liu2025flow, xue2025dancegrpo}.
However, these methods rely on costly human annotations or computationally expensive VLM-based reward models. Moreover, the supervision they provide is inherently global and holistic, introducing substantial overhead into the training pipeline while offering little improvement in fine-grained local details.

In this work, we argue that the intrinsic manifold structure of high-quality Supervised Fine-Tuning (SFT) data inherently acts as a cost-free reward model. Because modern Diffusion Transformers (DiTs) \cite{peebles2023scalable, ma2024latte} operate by flattening video latents into spatio-temporal patches, we propose to explicitly model the manifold at this patch level using the SFT data. This patch manifold provides dense, localized reward signals that are highly sensitive to fine-grained local details, ensuring that structural regularities, such as spatial coherence, motion consistency, and texture statistics, are strictly preserved. Consequently, we can enhance video generation simply by penalizing geometric deviations from this patch manifold, eliminating the need for external supervision.

Inspired by manifold learning \cite{roweis2000nonlinear, humayun2024secrets, humayun2024local}, we model the distribution of these high-quality SFT video patches using Local Coordinate Coding (LCC) \cite{yu2009nonlinear}. LCC approximates a non-linear manifold using a set of local anchor points, enabling a robust and continuous representation of the valid latent space. Unlike discrete vector quantization \cite{yan2021videogpt}, LCC approximates data via local linear combinations, ensuring that the learned representation respects the intrinsic geometry of the video patches.

However, it is non-trivial to apply LCC for high-dimensional manifold learning. First, we find that LCC tends to approximate local data neighborhoods through their geometric mean. Such local means reside in low-density regions according to the Gaussian Annulus Theorem \cite{vershynin2018high}, which states that the true probability mass of high-dimensional data concentrates on an outer shell. Second, LCC fails to account for the varying scales of different dimensions, thereby ignoring their diverse importance. To address these issues, we propose a Shell Local Coordinate Coding (Shell-LCC) method that pushes generated samples towards the high-density shell regions and adjusts dimensional importance to derive an isotropic shell of the data manifold. By doing this, we can accurately characterize the high-density regions of the true data distribution.

Building on Shell-LCC, we derive a dense and differentiable geometric reward that encourages generated latent patches to lie on the learned SFT manifold. This geometric constraint effectively improves realism, enhances high-frequency details, reduces over-smoothing artifacts, and alleviates motion blur. It is worth noting that, unlike standard SFT, our reward operates on text prompts and does not require paired text-video. Furthermore, whereas standard SFT enforces a rigid, point-to-point alignment by penalizing deviations from a specific ground-truth video representation, our method imposes a flexible point-to-surface constraint by optimizing the distance to the manifold. This point-to-surface alignment ensures the generated outputs maintain high structural quality.
Our contributions can be summarized as follows:

\begin{itemize}
    \item We demonstrate that the intrinsic manifold of SFT data can serve as a cost-free reward model. By constructing this manifold at the spatio-temporal patch level, we derive dense, differentiable, and nearly cost-free reward signals that eliminate the heavy computational overhead and annotation costs of traditional RLHF methods.
    \item We introduce Shell-LCC, a novel manifold modeling approach. By addressing the mean regression problem of standard LCC and modeling the data surface as an isotropic shell, Shell-LCC successfully prevents the loss of high-frequency structural details.
    \item Extensive experiments on both proprietary and open-source T2V models show that Shell-LCC improves perceptual realism and fine-grained imaging quality while suppressing low-level distortions, all without sacrificing semantic alignment or temporal consistency.
\end{itemize}

\section{Related Work}

\subsection{Text-to-Video Generation}

Diffusion-based models have emerged as the dominant paradigm for Text-to-Video (T2V) generation. Early seminal works successfully extended image diffusion models into the temporal domain by integrating spatio-temporal attention mechanisms and 3D neural architectures \cite{ho2022video, singer2022make, blattmann2023stable}. Recently, the field has witnessed a paradigm shift towards Diffusion Transformers (DiTs) operating within the compressed latent space of video autoencoders (VAEs) \cite{peebles2023scalable, brooks2024video, zheng2024open}, offering unprecedented scalability and generation fidelity. Despite these rapid architectural advancements, current training paradigms predominantly focus on pure data fitting, specifically matching the target data distribution by minimizing the predictive error of noise or velocity trajectories \cite{song2020score, lipman2022flow}. However, these likelihood-based and flow-matching objectives treat the latent representations as residing in a flat Euclidean space, lacking explicit mechanisms for the rigorous construction or geometric regularization of the underlying data manifold. This absence of structural constraint leaves the generation and fine-tuning processes prone to deviating from the true video manifold, particularly when adapting pre-trained models with limited data.

\subsection{Post-Training and Alignment}

Post-training techniques, such as RLHF and DPO, have become essential for aligning generative models with human expectations \cite{christiano2017deep, rafailov2023direct, zhao2023slic, munos2024nash}. While these methods have shown promise in improving visual quality and text-video alignment \cite{huang2024vbench, han2025video, sun2025t2v}, they typically rely on external reward models or heuristic metrics that are prone to bias, noise, and reward hacking. Furthermore, such alignment processes often require substantial preference data, which is difficult to curate for specialized video domains. In contrast, our approach redefines the alignment objective through a geometric lens: by deriving guidance signals directly from a constructed data manifold, it acts as an intrinsic regularization mechanism that keeps the model within the high-density regions of the data distribution. This avoids the instability of external supervision and provides robust, self-supervised alignment signals during the supervised fine-tuning (SFT) phase, especially in data-constrained scenarios \cite{brooks2024video, wang2025lavie, kondratyuk2023videopoet, bar2024lumiere}.

\subsection{Manifold Learning in Representation Spaces}

While classical manifold learning methods, particularly LCC, effectively preserve intrinsic geometric structures via translation-invariant local approximations \cite{roweis2000nonlinear, tenenbaum2000global, belkin2003laplacian, yu2009nonlinear, saul2003think}, their direct integration as geometric regularizers in the SFT of Diffusion Transformers remains unexplored. Recent studies further demonstrate that explicit manifold constraints significantly enhance the stability and expressivity of deep generative models \cite{de2022riemannian, arvanitidis2017latent, humayun2024secrets, humayun2024local, dai2021manifold, ni2022manifold}, yet bridging this gap toward DiT fine-tuning still requires tailored structural guidance. Beyond generative modeling, a related line of work structures the representation space for discriminative tasks such as deep regression~\cite{zhang2023improving, zhang2025improving}, among which topological and manifold-based regularization~\cite{zhang2024deep} in particular resonates with our geometric perspective. In this work, we utilize the local linearity of LCC to provide explicit manifold guidance, ensuring the model captures the fine-grained spatio-temporal structures essential for video generation. Furthermore, distinct from traditional LCC approaches, we leverage the Gaussian Annulus Theorem to approximate data specifically within high-probability outer shell regions, enabling the generation of realistic samples with rich high-frequency details.

\section{Preliminaries}

\subsection{Latent Video Diffusion Models}
Text-to-Video (T2V) generation is commonly formulated within the framework of Latent Diffusion Models (LDMs).
Given a video in the pixel space $\mathbf{x} \in \mathbb{R}^{T \times H \times W \times 3}$, where $T$, $H$, and $W$ denote the number of frames, height, and width, a pre-trained 3D autoencoder $\mathcal{E}$ first compresses $\mathbf{x}$ into a lower-dimensional latent space, yielding $\mathbf{z}_0 = \mathcal{E}(\mathbf{x}) \in \mathbb{R}^{T' \times H' \times W' \times C}$, with $T'$, $H'$, $W'$ the downsampled spatio-temporal dimensions and $C$ the number of latent channels.

Instead of treating $\mathbf{z}_0$ merely as a dynamic sequence of frames, modern transformer-based LDMs (e.g., DiT) process this latent volume as a dense representation comprising a sequence of spatio-temporal patches. We flatten and project $\mathbf{z}_0$ into a set of dense latent patches $\mathbf{Z}_0 = \{z_{0,1}, z_{0,2}, \dots, z_{0,N}\}$, where $N$ is the total number of patches and $z_{0,i} \in \mathbb{R}^d$ captures the rich local spatio-temporal semantics.

The diffusion process progressively adds Gaussian noise to the latent representation over a diffusion timestep $\tau \in [0, 1]$. The forward process is defined as:
\begin{equation}
    q(\mathbf{Z}_\tau | \mathbf{Z}_0) = \mathcal{N}(\mathbf{Z}_\tau; \sqrt{\bar{\alpha}_\tau}\mathbf{Z}_0, (1-\bar{\alpha}_\tau)\mathbf{I})
\end{equation}
where $\bar{\alpha}_\tau$ dictates the noise schedule. During training, a neural network $\epsilon_\theta$, typically a transformer, is trained to predict the added noise $\epsilon$ conditioned on text prompts $\mathbf{c}$. The standard SFT objective is the mean squared error of the noise prediction:
\begin{equation}
    \mathcal{L}_{SFT} = \mathbb{E}_{\mathbf{Z}_0, \epsilon, \tau, \mathbf{c}} \left[ \| \epsilon - \epsilon_\theta(\mathbf{Z}_\tau, \tau, \mathbf{c}) \|_2^2 \right]
\end{equation}
Following Flow Matching and Rectified Flows \cite{lipman2022flow, liu2022flow, esser2024scaling},
the parameterization is often reformulated from noise prediction to velocity prediction, where the network learns a time-dependent vector field $v_{\theta}(\mathbf{Z}_\tau, \tau, \mathbf{c})$ that describes the transport direction along a predefined interpolation path between data and noise distributions.

While this objective effectively matches the marginal distribution of the SFT data, it inherently operates on individual noise perturbations and struggles to explicitly regularize the intrinsic geometric structure of these dense representations, leading to potential blurring or out-of-distribution artifacts during generation. Consequently, the learned model may produce samples that are statistically plausible yet geometrically inconsistent with the underlying data manifold.

\subsection{Local Coordinate Coding (LCC)}
To geometrically regularize the T2V generation process, we hypothesize that the dense latent patches $z_i \in \mathbb{R}^d$ of high-quality, realistic videos reside on a low-dimensional non-linear manifold $\mathcal{M} \subset \mathbb{R}^d$. To explicitly model this structure, we employ LCC \cite{yu2009nonlinear}.

Unlike global linear methods such as PCA, LCC approximates the non-linear manifold by learning a dictionary of local anchor points (bases) $\mathcal{B} = \{b_1, b_2, \dots, b_M\}$, where $b_m \in \mathbb{R}^d$. For any given dense patch representation $z \in \mathcal{M}$, LCC seeks a sparse coordinate vector $\gamma \in \mathbb{R}^M$ to reconstruct $z$ as a linear combination of its locally adjacent anchors. The LCC objective $\mathcal{L}_{\text{LCC}}$ is formulated as:
\begin{equation}
\label{eq:lcc}
    \min_{\mathcal{B}, \gamma} \sum_{i=1}^N \left( \left\| z_i - \sum_{m=1}^M \gamma_{i,m} b_m \right\|_2^2 + \lambda \sum_{m=1}^M |\gamma_{i,m}| \| z_i - b_m \|_2^2 \right),
\end{equation}
subject to the shift-invariant constraint $\sum_{m=1}^M \gamma_{i,m} = 1$, where $\gamma_{i,m}$ is the $m$-th coordinate for patch $z_i$, and $\lambda > 0$ is the regularization parameter controlling the locality.
In our implementation, we parameterize the coordinates with a softmax function, which additionally enforces $\gamma_{i,m} \geq 0$ and thus yields a convex combination of nearby anchors. This design stabilizes optimization and is commonly used in deep local coding frameworks.

The first term of Eq.~\ref{eq:lcc} represents the reconstruction error. The crucial second term is the \textbf{locality constraint}: it penalizes the $\ell_1$ norm of the coding weighted by the squared Euclidean distance between the data point $z_i$ and the anchor $b_m$. This ensures that $\gamma_{i,m}$ is non-zero only for anchors $b_m$ that are geometrically close to $z_i$.

Consequently, the localized linear approximation $\hat{z} = \sum_{m=1}^M \gamma_{m} b_m$ defines a local tangent space around $z$ on the manifold $\mathcal{M}$. 
This representation provides a reconstruction of manifold-consistent points and enables measuring the distance of a patch to the manifold.

\section{Methodology}

In this section, we construct our differentiable geometric reward framework. We first present a theoretical analysis showing why standard manifold approximations lead to mean regression and loss of high-frequency details in dense representations. To resolve this, we propose Shell Local Coordinate Coding (Shell-LCC) and derive the corresponding dense differentiable geometric reward.

\subsection{The Mean Regression Dilemma in Local Coordinate Coding}
\label{sec:mean_regression}

While LCC effectively captures non-linear geometric structure, directly using it as a reconstruction constraint in generative modeling introduces an inherent over-smoothing effect. We formalize this phenomenon as a \textit{mean regression} property of convex local reconstruction.

\begin{lemma}[Variance Decomposition]
\label{lemma:variance}
Assume $\gamma_m \geq 0$ for all $m$ and $\sum_{m=1}^M \gamma_m = 1$, and let $\hat{\mathbf{z}} = \sum_{m=1}^M \gamma_m \mathbf{b}_m$ be the reconstructed point. The locality penalty term decomposes as:
\begin{equation}
    \sum_{m=1}^M \gamma_m \|\mathbf{z} - \mathbf{b}_m\|^2 = \|\mathbf{z} - \hat{\mathbf{z}}\|^2 + \sum_{m=1}^M \gamma_m \|\mathbf{b}_m - \hat{\mathbf{z}}\|^2
\end{equation}
\end{lemma}
\begin{proof}
Expanding the left-hand side:
\begin{align*}
    \sum_{m=1}^M \gamma_m \|\mathbf{z} - \mathbf{b}_m\|^2
    &= \sum_{m=1}^M \gamma_m \left( \|\mathbf{z} - \hat{\mathbf{z}}\|^2 - 2 \langle \mathbf{z} - \hat{\mathbf{z}},\, \mathbf{b}_m - \hat{\mathbf{z}} \rangle + \|\mathbf{b}_m - \hat{\mathbf{z}}\|^2 \right).
\end{align*}
Since $\sum_{m=1}^M \gamma_m = 1$, we have $\sum_{m=1}^M \gamma_m (\mathbf{b}_m - \hat{\mathbf{z}}) = \hat{\mathbf{z}} - \hat{\mathbf{z}} = 0$, so the cross-term vanishes.
\end{proof}

\begin{theorem}[Mean Regression in LCC]
Under the assumption of $\gamma_m(z) \ge 0,\ \sum_m \gamma_m(z) = 1$, minimizing the LCC objective in Eq.~\ref{eq:lcc} is equivalent to minimizing
\begin{align}
\min_{\mathcal{B}, \gamma} \sum_{i=1}^N \left(\left(1 + \lambda\right) \|\mathbf{z}_i - \hat{\mathbf{z}}_i\|^2 + \lambda \sum_{m=1}^M \gamma_{i,m} \|\mathbf{b}_m - \hat{\mathbf{z}}_i\|^2 \right)
\end{align}
\end{theorem}
\begin{proof}
Let $\mathcal{L}(\mathcal{B}, \gamma)$ denote the LCC objective in Eq.~\ref{eq:lcc}. Under $\gamma_{i,m} \ge 0$ and $\sum_m \gamma_{i,m} = 1$ (so $|\gamma_{i,m}| = \gamma_{i,m}$), applying Lemma~\ref{lemma:variance} to the locality term gives:
\begin{align}
    \mathcal{L}(\mathcal{B}, \gamma)
    &= \sum_{i=1}^N \left( \|\mathbf{z}_i - \hat{\mathbf{z}}_i\|^2 + \lambda \sum_{m=1}^M \gamma_{i,m} \|\mathbf{z}_i - \mathbf{b}_m\|^2 \right) \\
    &= \sum_{i=1}^N \left( \|\mathbf{z}_i - \hat{\mathbf{z}}_i\|^2 + \lambda \left(\|\mathbf{z}_i - \hat{\mathbf{z}}_i\|^2 + \sum_{m=1}^M \gamma_{i,m} \|\mathbf{b}_m - \hat{\mathbf{z}}_i\|^2\right) \right) \\
    &= \sum_{i=1}^N \left((1 + \lambda)\|\mathbf{z}_i - \hat{\mathbf{z}}_i\|^2 + \lambda \sum_{m=1}^M \gamma_{i,m} \|\mathbf{b}_m - \hat{\mathbf{z}}_i\|^2 \right)
\end{align}
\end{proof}

Consequently, the optimization simultaneously:

\begin{enumerate}
\item pulls $\hat{\mathbf{z}}_i$ toward the data point $\mathbf{z}_i$ via reconstruction error;
\item pulls $\hat{\mathbf{z}}_i$ toward the weighted centroid of anchors by minimizing local variance.
\end{enumerate}

Therefore, the optimal reconstruction is necessarily a compromise between these two forces, resulting in a systematic shrinkage of $\hat{\mathbf{z}}_i$ toward the local anchor mean. This establishes the mean regression effect of LCC. As a result, directly enforcing LCC reconstruction in generative training inevitably introduces over-smoothed latent representations, leading to blurred textures and loss of fine details.

\subsection{The `Empty Core' Paradox in Convex Reconstruction}
\label{sec:empty_core}

To counteract the mean regression effect of LCC, it is crucial to understand the true geometric structure of high-dimensional latent representations.
In high-dimensional spaces, probability mass does not concentrate near the mean, but instead concentrates on a thin ``shell'' at a distance from the mean. This phenomenon is formalized by the Gaussian Annulus Theorem.

\begin{theorem}[Gaussian Annulus Theorem \cite{vershynin2018high}]
Let $x \sim \mathcal{N}(0, I_d)$ in $\mathbb{R}^d$. Then for any $\epsilon \in [0,\sqrt{d}]$:
\[
\Pr\Big(
\big| \|x\| - \sqrt{d} \big| \ge \epsilon 
\Big)
\le 2 e^{-c \epsilon^2 }
\]
for some universal constant $c > 0$.
\end{theorem}

\begin{corollary}[Zero Probability Mass at the Mean]
\label{cor:empty_core}
If the local data distribution in a $d$-dimensional dense representation space can be approximated by an isotropic Gaussian, then the probability mass within any small spherical neighborhood around the local centroid $\mu$ approaches zero exponentially as $d$ increases.
\end{corollary}

\begin{proof}
Let $x \sim \mathcal{N}(\mu, I_d)$ in $\mathbb{R}^d$, and consider a small spherical neighborhood of radius $r>0$ around the centroid $\mu$:
\[
B(\mu, r) = \{ x \in \mathbb{R}^d : \|x-\mu\| \le r \}.
\]
Translate to the origin: define $y = x - \mu \sim \mathcal{N}(0, I_d)$. Then
\[
\Pr(x \in B(\mu, r)) = \Pr(\|y\| \le r).
\]
For any fixed $r>0$ and sufficiently large $d$ (so that $r < \sqrt{d}$), set $\epsilon = \sqrt{d} - r > 0$. Then
\[
\|y\| \le r = \sqrt{d} - \epsilon \implies |\|y\| - \sqrt{d}| \ge \epsilon.
\]
Applying the Gaussian Annulus Theorem gives
\[
\Pr(\|y\| \le r) \le \Pr\!\left(|\|y\| - \sqrt{d}| \ge \epsilon\right) \le 2e^{-c(\sqrt{d}-r)^2}.
\]
As $d \to \infty$ with fixed $r$, the right-hand side decays exponentially, so $\Pr(x \in B(\mu, r)) \to 0$.
\end{proof}

The Gaussian Annulus Theorem and this corollary
imply that as dimensionality increases, almost all probability mass lies within a thin shell:
\[
\|x\| \approx \sqrt{d},
\]
while the region near the mean has exponentially small probability.

\textbf{Implication for latent representations.} Although latent features are not strictly Gaussian, high-dimensional embeddings are well known to exhibit approximately Gaussian local statistics due to central limit effects and concentration of measure. Therefore, within a local neighborhood, latent patches behave similarly: their probability mass concentrates on a thin shell around the local mean rather than at the mean itself. Consequently, the true data distribution is effectively hollow in the radial direction.

\subsection{Shell Local Coordinate Coding (Shell-LCC)}
\label{sec:shell_lcc}

To accurately capture the intrinsic geometry of the SFT data manifold, we must move beyond the central ``skeleton'' approximated by standard LCC. Properly modeling this structure requires addressing two fundamental properties of high-dimensional latent spaces: the dimensional imbalance of the features and the high-dimensional concentration of measure.

\textbf{Isotropic Calibration.} The latent dimensions of $z \in \mathbb{R}^d$ are not strictly isotropic; they encode features with varying scales and semantic importance. Standard LCC relies on unweighted Euclidean distance, which inherently biases the optimization towards high-variance dimensions while neglecting fine-grained textures. To calibrate this, we map the local geometry into an isotropic space using a diagonal Mahalanobis distance. Let $\sigma = [\sigma_1, \dots, \sigma_d]^\top \in \mathbb{R}_{>0}^d$ be a learnable scale vector. For a latent patch $z$ and its local linear reconstruction $\hat{z}$, the dimension-scaled residual is defined as:
\begin{equation}
    \tilde{z} = \Sigma^{-\frac{1}{2}}(z - \hat{z})
\end{equation}
where $\Sigma = \text{diag}(\sigma_1^2, \dots, \sigma_d^2)$. This scaling ensures the manifold approximation respects low-variance but structurally crucial features equally.

\textbf{The Shell Constraint.} Even with an isotropically scaled space, standard LCC encourages $\|\tilde{z}\|_2 \to 0$, pulling the patch towards the local mean. However, according to the Gaussian Annulus Theorem, the true probability mass in high-dimensional spaces concentrates on a thin outer shell rather than the geometric center. To preserve spatial variance and align strictly with this high-density region, we restrict the latent codes to reside on a spherical shell of radius equal to 1 (representing a unit standard deviation boundary). The distance penalty is formulated as:
\begin{equation}
    \mathcal{L}_{\text{shell\_dist}} = (\|\tilde{z}\|_2 - 1)^2
\end{equation}

\textbf{Overall Objective.} Minimizing $\mathcal{L}_{\text{shell\_dist}}$ alone introduces a trivial solution: rather than learning an accurate manifold representation, the model could simply manipulate the scale vector $\sigma$ to artificially satisfy $\|\tilde{z}\|_2 = 1$ for any poor reconstruction, rendering the geometric constraint meaningless. To prevent the variances from growing unconstrained to compensate for large reconstruction errors, we introduce a logarithmic regularizer. This mathematically aligns our objective with minimizing the negative log-likelihood of a shell-structured multivariate Gaussian, forcing the model to learn a tight and structurally faithful envelope. The complete Shell-LCC objective is:
\begin{equation}
    \mathcal{L}_{\text{Shell-LCC}} = \mathcal{L}_{\text{LCC}} + \tau_1 \mathcal{L}_{\text{shell\_dist}} + \tau_2 \sum_{i=1}^d \log \sigma_i
\end{equation}

By optimizing this objective, the learnable $\sigma$ automatically calibrates dimensional importance, while the shell constraint actively preserves high-frequency structural details, effectively curing the mean regression issue. We empirically set $\tau_1 = 1$ and $\tau_2 = 0.1$.

\subsection{Dense Differentiable Geometric Rewards}
\label{sec:rewards}

With the Shell-LCC effectively modeling the true geometry of the SFT data, we freeze it and derive a differentiable reward signal to fine-tune the T2V model. Let $z_{gen}$ be a dense latent patch generated by the model, and $\hat{z}_{gen} = \sum_{m=1}^M \gamma_m b_m$ be its local linear reconstruction, i.e., its projection onto the learned Shell-LCC manifold.

Typically, the generated patch $z_{gen}$ exhibits a significantly larger orthogonal distance from its reconstruction $\hat{z}_{gen}$ than real patches do. This occurs because the SFT anchors $b_m$ only span the subspace of valid, high-quality video structures. Consequently, any out-of-distribution artifacts, physically implausible movements, or low-level noise present in $z_{gen}$ cannot be accurately reconstructed. This un-spannable residual explicitly isolates the generative distortions we aim to penalize.

To ensure spatial fidelity, one might intuitively minimize this distance to zero. However, as dictated by the Gaussian Annulus Theorem, the exact reconstruction $\hat{z}_{gen}$ represents the local geometric mean, which resides in a low-density region. Continuously minimizing the distance to zero would force $z_{gen}$ directly into this `empty' center, inevitably inducing severe mean regression and yielding blurred, over-smoothed videos.

To penalize artifacts while preserving high-frequency sharpness, we only pull $z_{gen}$ inward until it reaches the high-density manifold shell. Using our calibrated Mahalanobis space, the distance reward is formulated to minimize the deviation from the Shell-LCC manifold:
\begin{equation}
    R_{dist}(z_{gen}) =  
    \left\| \Sigma^{-\frac{1}{2}}(z_{gen} - \hat{z}_{gen}) \right\|_2^2
\end{equation}
minimizing $R_{dist}$ ensures spatial fidelity without inducing the blurring effect associated with standard mean-regression penalties.

A major advantage of our Shell-LCC formulation is that the geometric reward $R_{dist}$ is fully differentiable with respect to the dense representation $z_{gen}$. This eliminates the need for high-variance reinforcement learning algorithms (e.g., PPO) or auxiliary reward models, allowing for direct gradient backpropagation into the T2V diffusion backbone.

\section{Experiments}

This section presents a comprehensive evaluation of our approach. We first validate Shell-LCC's ability to capture the underlying manifold structure of SFT data, followed by its effectiveness in improving T2V realism and aesthetics. 
We mainly conduct experiments using our private 4.5B T2V model on our private SFT data, and also evaluate on the publicly available Wan-T2V-1.3B model \cite{wan2025} with the UltraVideo dataset \cite{ultravideo}.

\begin{figure}[t]
    \centering
    \begin{minipage}[c]{0.52\textwidth}
        \centering
        \includegraphics[width=0.95\linewidth]{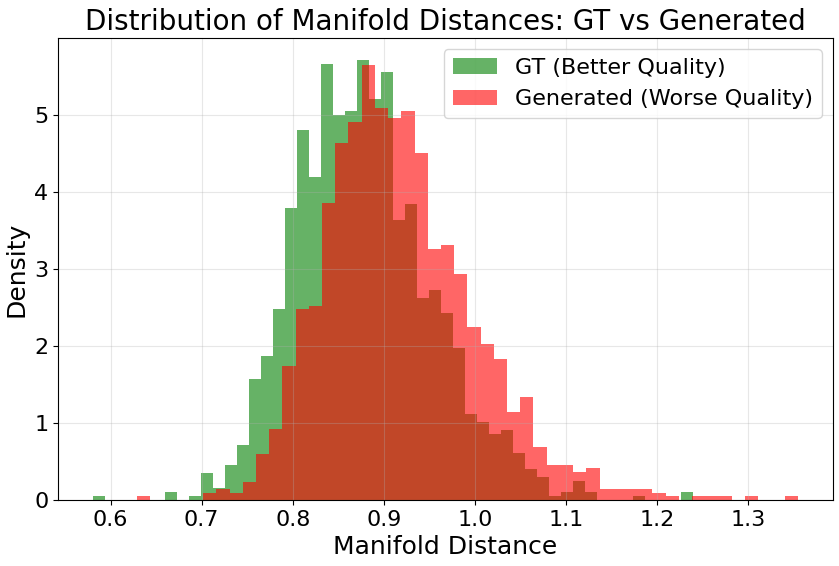}
        \caption{Distribution of the manifold distance $R_{dist}$ for Ground Truth and generated videos. Generated videos show a clear distribution shift, indicating manifold drift.}
        \label{fig:manifold_dist}
    \end{minipage}\hfill
    \begin{minipage}[c]{0.45\textwidth}
        \centering
        \makeatletter\def\@captype{table}\makeatother 
        \caption{Manifold distinguishability on the test set. We report the Accuracy (Acc \%) where $R_{dist}(z_{gt}) < R_{dist}(z_{gen})$.}
        \label{tab:distinguish_acc}
        \resizebox{\linewidth}{!}{
        \begin{tabular}{lcc}
            \toprule
            \textbf{Method} & \textbf{Epoch} & \textbf{Acc (\%)} \\
            \midrule
            LCC & 10  & 92.1 \\
            LCC & 100 & 92.3 \\
            LCC & 200 & 92.7 \\
            \midrule
            Recon Only  & 200 & 83.2 \\
            Shell-LCC  & 200 & 91.4 \\
            \bottomrule
        \end{tabular}
        }
    \end{minipage}
\end{figure}

\subsection{Shell-LCC Training and Geometric Structure Analysis}

\textbf{Shell-LCC Configuration and Training Details.}
Our Shell-LCC is configured with a learnable basis matrix $B \in \mathbb{R}^{M \times d}$ containing $M = 4096$ local anchor points, which are initialized by randomly selecting SFT video patches $z$. The coordinate predictor is implemented as a 3-layer Multi-Layer Perceptron (MLP) with ReLU activations, culminating in a Softmax layer to enforce sparsity in the local coordinate assignments. The learnable scale vector $\sigma$ is also implemented as a 3-layer MLP. 
During the optimization phase, the model is trained using the AdamW optimizer with a learning rate of $1 \times 10^{-3}$. 

\textbf{Manifold Distance and Distinguishability of $R_{dist}$.}
To verify that Shell-LCC accurately models the SFT data manifold and that $R_{dist}$ provides a meaningful reward signal, we evaluate its ability to distinguish between high and low-quality videos. We split our SFT dataset into 2,500 training and 1,501 test samples. Each pair comprises a high-quality ground-truth (GT) video and a relatively lower-quality video generated by our baseline T2V model trained on the SFT dataset.

A valid manifold representation should consistently yield $R_{dist}(z_{\text{gt}}) < R_{dist}(z_{\text{gen}})$. As shown in Table~\ref{tab:distinguish_acc}, the distinguishability accuracy of LCC steadily increases across epochs, reaching 92.7\% at Epoch 200, while Shell-LCC reaches 91.4\%. This confirms that both LCC (where $R_{dist} = \|z - \hat{z}\|_2^2$) and Shell-LCC successfully learn the intrinsic manifold structure over time, thereby validating $R_{dist}$ as an effective geometric reward signal to penalize generative distortions.

\textbf{Analyzing the Distinguishability Trade-off.} 
Interestingly, the discriminative accuracy of Shell-LCC is marginally lower than that of LCC. This directly results from our isotropic calibration. Standard LCC uses an unweighted Euclidean distance dominated by high-variance dimensions, making generated latents with massive trivial errors trivially easy to classify. By introducing the scale vector $\sigma$ to normalize the residual distance, Shell-LCC dampens these high-variance dimensions. Although this sacrifices a fraction of naive binary accuracy, it strictly enforces equal penalization of low-variance, fine-grained structural details, thereby prioritizing true geometric fidelity over trivial variance exploitation.

\textbf{The importance of the local constraint in LCC.}
To demonstrate the importance of the local constraint in LCC (i.e., $\sum_{m=1}^M |\gamma_{i,m}| \| z_i - b_m \|_2^2$ in Eq. \ref{eq:lcc}), we train the manifold only with the standard reconstruction loss (i.e., $z_i - \hat{z}_i$). As presented in the bottom row of Table~\ref{tab:distinguish_acc}, its ability to distinguish between GT and generated videos significantly degrades ($-9.5\%$), although the absolute reconstruction error can be minimized to a near-zero value. This phenomenon is highly intuitive: prioritizing pure reconstruction allows the model to degenerate into a trivial identity mapping. Consequently, the latent space collapses, and the model merely memorizes the input without learning the meaningful, low-dimensional geometric structure of the video manifold.

\textbf{Empirical Evidence of Shell Geometry.} To validate our hypothesis, we analyze the normalized radial distances of high-quality (real) and low-quality (generated) latent patches. As shown in Fig.~\ref{fig:manifold_dist}, we observe a distinct separation in their radial statistics: real samples average at $0.880 \pm 0.076$, while generated samples average at $0.918 \pm 0.083$.  This reveals three critical phenomena:
\begin{itemize}
    \item \textbf{Non-zero Concentration:} Real samples cluster at a stable, non-zero radius rather than the local mean, contradicting standard LCC assumptions but strongly aligning with the Gaussian Annulus Theorem.
    \item \textbf{Thin Manifold Shell:} The tight radial variance ($0.076$) confirms that realistic video latents are confined to a narrow geometric band.
    \item \textbf{Generative Drift:} Generated samples systematically exhibit larger distances and variances, quantitatively reflecting their deviation from the true data manifold.
\end{itemize}
Although the exact empirical radius is slightly below $1.0$ due to the logarithmic regularizer, the existence of this stable, discriminative shell provides solid justification for our point-to-surface reward design.

\textbf{Hyperparameter Sensitivity.} Shell-LCC is insensitive to the anchor count $M$ and loss weights $\tau_1, \tau_2$. As shown in Table~\ref{tab:ablation_hyperparams}, $M{=}2048$ performs on par with $4096$, and varying $\tau_1$ or $\tau_2$ by two orders of magnitude changes accuracy by less than $0.3\%$. This robustness to $M$ is expected: the anchors only need to capture the coarse manifold ``skeleton'' rather than provide dense coverage, while the learnable scale $\sigma$ further enriches the representation. Moreover, the patch-level formulation feeds a large number of patches into a low-dimensional ($16$-d) manifold, so point-wise anchor precision is unnecessary; patch-level gradient noise cancels out over batched updates, yielding a stable optimization direction.

\begin{table}[htbp]
\centering
\scriptsize
\setlength\tabcolsep{4pt}
\caption{Ablation on Shell-LCC hyperparameters and data scale (\#Videos). Bold denotes our default setting.}
\label{tab:ablation_hyperparams}
\resizebox{\linewidth}{!}{
\begin{tabular}{l|ccc|ccc|ccc|ccc}
\toprule
& \multicolumn{3}{c|}{$M$} & \multicolumn{3}{c|}{$\tau_1$} & \multicolumn{3}{c|}{$\tau_2$} & \multicolumn{3}{c}{\#Videos} \\
& 2048 & \textbf{4096} & 8192 & 0.1 & \textbf{1.0} & 10 & 0.01 & \textbf{0.1} & 1 & 100 & 1000 & \textbf{2500} \\
\midrule
Acc (\%) & 90.8 & 91.4 & 91.8 & 91.3 & 91.4 & 91.1 & 91.1 & 91.4 & 91.3 & 90.1 & 91.4 & 91.4 \\
\bottomrule
\end{tabular}
}
\end{table}

\textbf{Data Scale.} Our patch-wise manifold extracts 230,400 patches from a single 5-second video, so even a small video collection already provides a large number of training samples. As shown in Table~\ref{tab:ablation_hyperparams}, 100 videos achieve 90.1\% Acc, while 1,000 videos already reach 91.4\%---matching the full 2,500-video set.

\textbf{Time and memory consumption.} Our Shell-LCC model is highly lightweight, containing only 2.5M parameters, a mere 0.2\% of the Wan-T2V-1.3B model and 0.056\% of our 4.5B SFT model. As a result, it introduces near-zero additional latency and memory overhead.

\subsection{Effectiveness of the Manifold Reward}


\begin{table*}[tp]
    \centering
    \caption{Quantitative comparison on VBench; higher is better for all metrics. Shell-LCC significantly reduces visual distortion (Imaging Quality) without trading off global aesthetics, semantic alignment, or temporal consistency.}
    \label{tab:vbench}
    \renewcommand{\arraystretch}{1.2}
    \setlength\tabcolsep{8.0pt}
    \resizebox{1.0\linewidth}{!}{
        \begin{tabular}{lcccccc}
        \toprule[0.1em]
        \makecell[l]{Models} & \makecell[c]{Aesthetic Quality} & \makecell[c]{Imaging Quality}
        & \makecell[c]{Overall Consistency}
        & \makecell[c]{Motion Smoothness}
        & \makecell[c]{Subject Consistency}\\
        \midrule
        Our SFT Baseline & 0.6724 $\pm$ 0.0030 & 0.7509 $\pm$ 0.0007& 0.2637 $\pm$ 0.0020 & 0.9891 $\pm$ 0.0004 & 0.9680 $\pm$ 0.0005   \\
        Our SFT Baseline + DPO & 0.6784 $\pm$ 0.0019 & 0.7397 $\pm$ 0.0020 & 0.2677 $\pm$0.0009 & 0.9840 $\pm$ 0.0006 & 0.9637 $\pm$ 0.0005  \\   
        Our SFT Baseline + Shell-LCC (Ours) & 0.6735 $\pm$ 0.0047 & 0.7631 $\pm$ 0.0059 & 0.2670 $\pm$ 0.0027  & 0.9900 $\pm$ 0.0003 & 0.9682 $\pm$ 0.0003  \\       
        \midrule
        Wan-T2V-1.3B \cite{wan2025} & 0.6244 $\pm$ 0.0101& 0.6629 $\pm$ 0.0101 & 0.2271 $\pm$ 0.0070& 0.9848 $\pm$ 0.0023& 0.9654 $\pm$ 0.0023   \\ 
        Wan-T2V-1.3B + Shell-LCC (Ours) & 0.6253 $\pm$ 0.0097 & 0.7037 $\pm$ 0.0345& 0.2292 $\pm$ 0.0059 & 0.9871 $\pm$ 0.0034 & 0.9715 $\pm$ 0.0034    \\ 
        \midrule
        UltraWan-T2V-1.3B \cite{ultravideo} & 0.5744 $\pm$ 0.0122 & 0.6756 $\pm$ 0.0043& 0.2288 $\pm$ 0.0022 & 0.9657 $\pm$ 0.0024  & 0.9225 $\pm$ 0.0007  \\
        UltraWan-T2V-1.3B + Shell-LCC (Ours) & 0.6299 $\pm$ 0.0086 & 0.7396 $\pm$ 0.0121& 0.2236 $\pm$ 0.0052& 0.9918 $\pm$ 0.0049& 0.9658 $\pm$ 0.0099  \\
        \bottomrule[0.1em]
        \end{tabular}
    }
\end{table*}

\textbf{Experimental Setup.}
To evaluate the efficacy of the proposed manifold reward, we conduct experiments on both our proprietary 4.5B SFT model and prominent open-source architectures. For the proprietary model, the Shell-LCC is trained using an internal SFT dataset consisting of 4,000 curated high-quality video clips characterized by complex spatio-temporal dynamics. To demonstrate the generalizability of our approach, we further evaluate the publicly available Wan-T2V-1.3B \cite{wan2025} and UltraWan-T2V-1.3B \cite{ultravideo} models. For these, the Shell-LCC is trained on the UltraVideo dataset \cite{ultravideo} (42,000 clips at 4K/8K resolution). Finally, we compare the Shell-LCC manifold reward against Direct Preference Optimization (DPO) using paired preference data from VideoGen-RewardBench \cite{liu2025improving}. 
We update the T2V model solely via the manifold reward, omitting the base diffusion loss.
Early stopping is triggered when the manifold penalty drops to $\sim$90\% of its initial value ($\sim$200 iterations).
It is worth mentioning that our reward mechanism decouples optimization from fixed text-video pairs: by fine-tuning with arbitrary prompts and no ground-truth videos, Shell-LCC enables latent interpolation and extrapolation, generalizing robustly beyond the SFT data distribution. This makes it especially valuable for larger models, where curating paired SFT data is the dominant bottleneck.
All models are trained and evaluated at a resolution of 720p.

\textbf{Evaluation Metrics.} 
We assess our method along three complementary axes from the VBench suite. For visual fidelity, we use \emph{Aesthetic Quality} (global composition and photorealism) and \emph{Imaging Quality} (low-level distortions such as blur, noise, and exposure artifacts). For semantic alignment, we report \emph{Overall Consistency} (video--text consistency). For temporal stability, we report \emph{Motion Smoothness} and \emph{Subject Consistency}. Imaging Quality serves as our primary geometric proxy; we hypothesize that if latent representations are successfully regularized toward the intrinsic manifold shell, the resulting decoded frames will exhibit significantly fewer low-level artifacts.

\textbf{Quantitative Results.} Table \ref{tab:vbench} summarizes our quantitative evaluations. Applied to our SFT baseline, Shell-LCC yields a notable 1.22\% absolute gain in Imaging Quality, reflecting enhanced low-level fidelity through the effective reduction of blur, noise, and motion artifacts. Concurrently, this reduction in low-level distortion mitigates the over-smoothed, artificial textures symptomatic of AI-generated videos, recovering realistic, fine-grained details. Crucially, Shell-LCC significantly reduces visual distortion without trading off global aesthetics; whereas DPO slightly improves Aesthetic Quality at the direct expense of degraded Imaging Quality (dropping to 73.97\%), our geometric constraint preserves local sharpness alongside a stable Aesthetic Quality (67.35\%).
Similar trends observed across the Wan-T2V and UltraWan architectures confirm the robust generalizability of our approach (see Fig.~\ref{fig:open_source} for qualitative results).

\textbf{Collaborative effects.} 
Shell-LCC and methods like DPO operate at \textbf{different granularities} and are \textbf{orthogonal in semantic alignment}: DPO uses video-level supervision for global semantic/stylistic alignment; Shell-LCC uses dense patch-level rewards to refine local structure without altering content. Tab.~\ref{tab:vbench} confirms this: DPO improves Aesthetic Quality (+0.60\%) but \emph{sacrifices} Imaging Quality ($-1.12\%$), while Shell-LCC targets exactly these distortions, improving Imaging Quality by +1.22\% and +4.08\% on our model and Wan-T2V-1.3B, respectively. They are therefore complementary.

\textbf{Semantic alignment and temporal consistency.}
\emph{Semantic alignment is preserved.} Shell-LCC refines local details with minimal impact on content (Fig.~\ref{fig:qualitative_comp}), so semantics are not sacrificed. An A/B test by an independent evaluation group confirms this: Shell-LCC significantly outperforms the SFT baseline on realism (score 108 vs.\ 100), while semantic alignment is unchanged (101 vs.\ 100).
\emph{Temporal consistency is improved.} 
Our patches are \textbf{spatio-temporal volumes}, so temporal coherence is embedded in the manifold definition.
Although the reward is patch-wise, gradients aggregate over a full batch, applying the same geometric constraint uniformly to spatial and temporal dimensions. Tab.~\ref{tab:vbench} confirms gains in overall consistency, motion smoothness, and subject consistency.

\subsection{Qualitative Analysis}

We provide visual comparisons to better understand the geometric effect of Shell-LCC. Several consistent phenomena are observed. More visualizations are provided in the supplementary material.

\begin{figure}[t]
    \centering
    \includegraphics[width=0.95\linewidth]{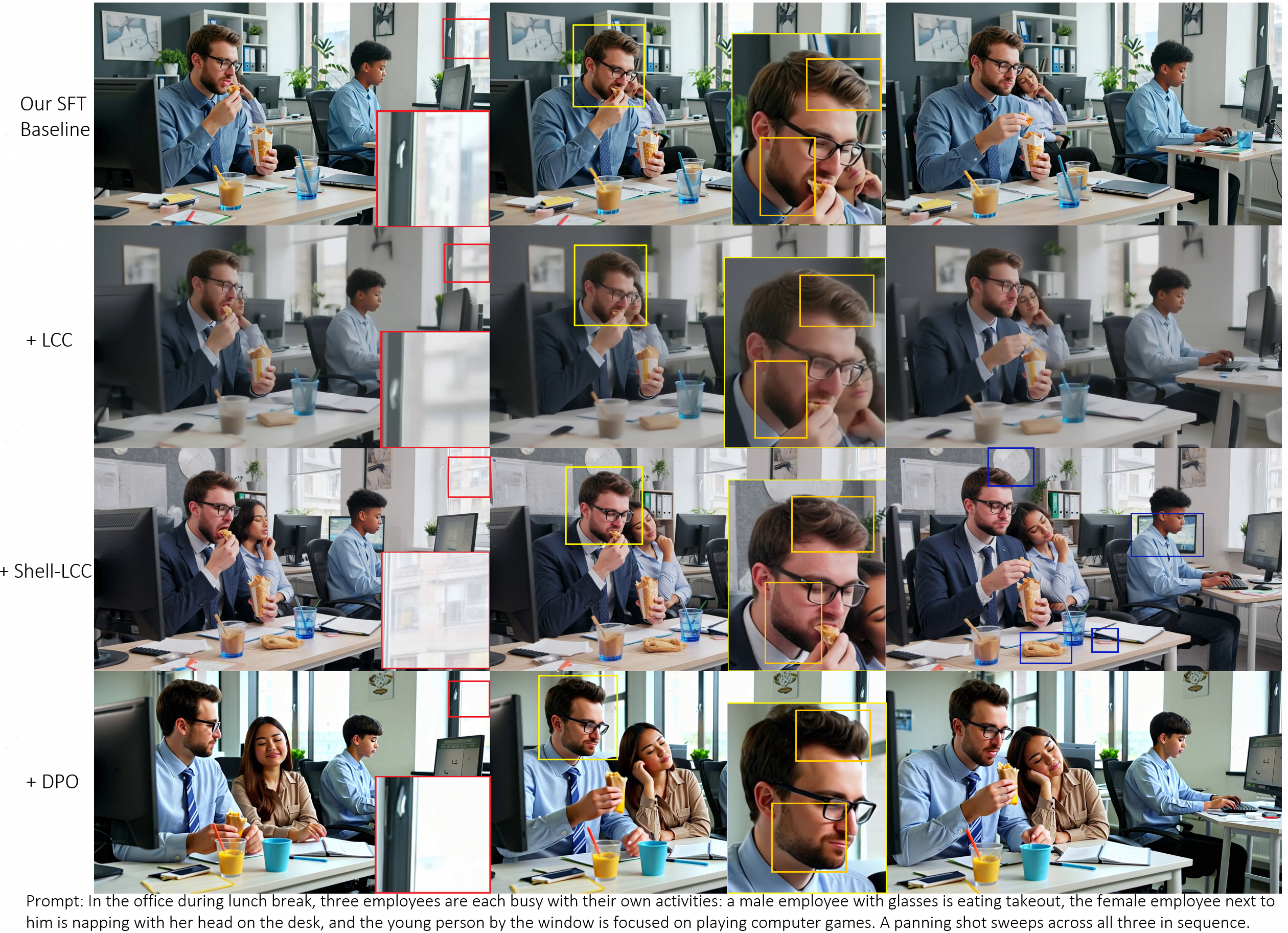}
    \caption{Qualitative comparison of generated videos. Methods from top to bottom: SFT baseline, +LCC, +Shell-LCC, and +DPO. As shown, adding Shell-LCC significantly improves video quality by mitigating low-level distortions. 
    Specifically, it yields cleaner backgrounds with richer high-frequency details (red boxes) and recovers realistic, fine-grained textures instead of the over-smoothed ``plastic'' look common in AI generation (yellow boxes). Blue boxes further highlight its ability to synthesize complex scenes with intricate local details. Compared to DPO, Shell-LCC better preserves dense visual structures, benefiting from dense reward signals.
    Additionally, unlike LCC, which is prone to mean regression and loses high-frequency information, our method successfully retains these details.}
    \label{fig:qualitative_comp}
\end{figure}

\textbf{High-frequency restoration for background clarity.}
As highlighted in the red boxes in Fig.~\ref{fig:qualitative_comp}, Shell-LCC significantly restores high-frequency information that is typically lost or blurred in SFT and LCC baselines. By providing dense, patch-level supervision, our method effectively deblurs complex backgrounds, transforming hazy textures into sharp, identifiable visual structures. This results in a much higher level of global-local consistency across the entire frame.

\textbf{Mitigation of over-smoothing for realistic textures.}
The yellow boxes in Fig.~\ref{fig:qualitative_comp} demonstrate that Shell-LCC alleviates the `plastic' or over-smoothed artifacts symptomatic of AI-generated videos. Unlike DPO, which can lead to uniform and artificial surfaces, Shell-LCC encourages the synthesis of realistic, fine-grained micro-textures. This ensures that object surfaces, such as skin, fabric, or natural elements, retain their authentic visual roughness and detail.

\textbf{Synthesis of complex scenes and local intricacy.}
As illustrated in the blue boxes in Fig.~\ref{fig:qualitative_comp}, Shell-LCC demonstrates a stronger capability to synthesize complex scenes with a higher density of local objects and fine-grained details. While baseline models often simplify cluttered environments or suppress small structures, the dense reward signals in Shell-LCC encourage the model to preserve the structural integrity of local components. As a result, the generated videos exhibit richer visual content and improved precision in representing intricate scene elements.

\begin{figure}[t]
    \centering
    \includegraphics[width=0.95\linewidth]{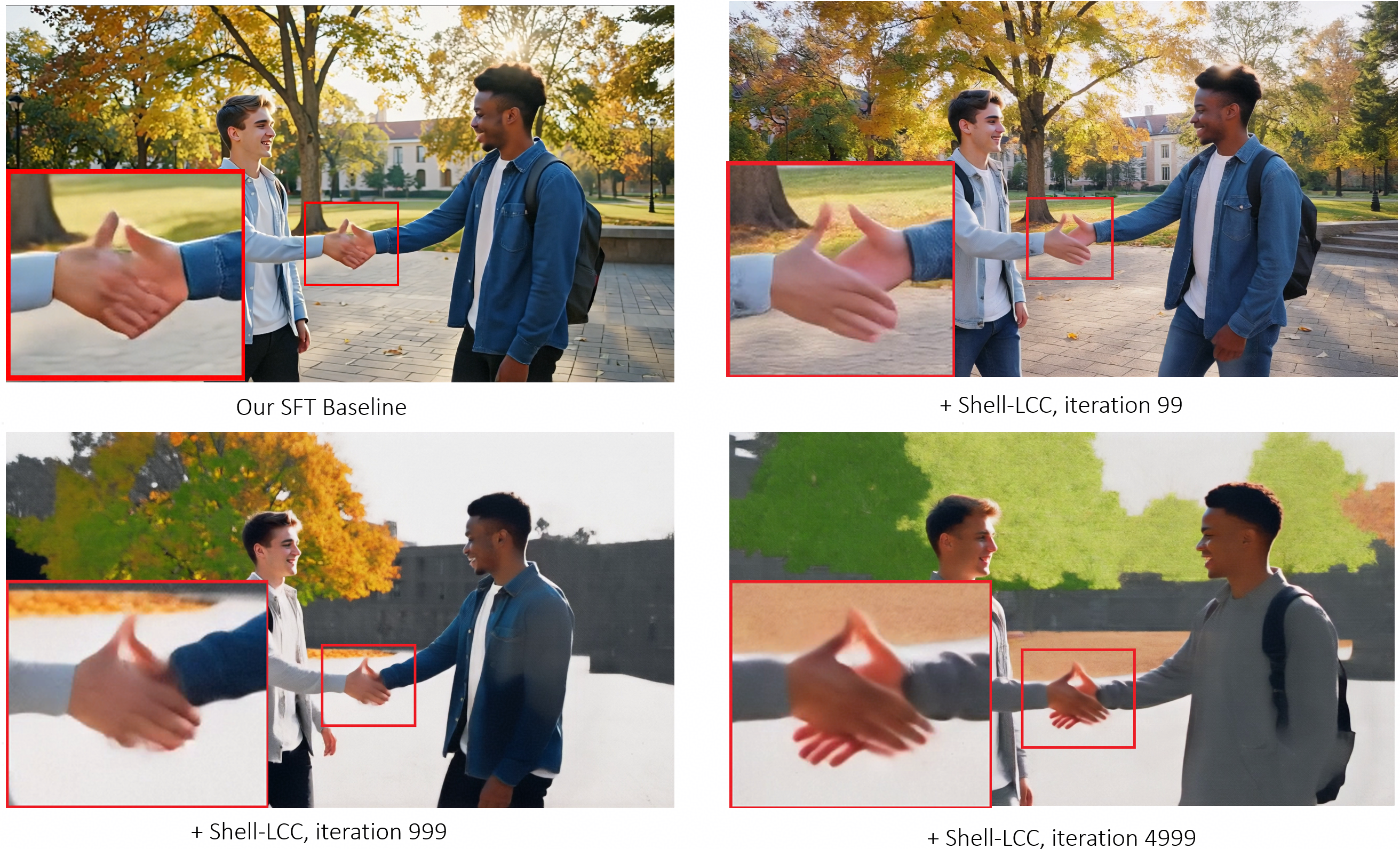}
    \caption{Evolution of motion deblurring and training stability. While Shell-LCC reduces the motion blur observed in the baseline, prolonged training (e.g., at iter. 4999) triggers a regression towards the mean ($\hat{z}$), leading to model collapse. This trade-off highlights the inherent conflict between eliminating out-of-manifold information (blur) and preserving high-frequency details, necessitating controlled optimization.}
    \label{fig:learning}
\end{figure}

\textbf{Controlled deblurring behavior and mean regression.}
As shown in Fig.~\ref{fig:learning}, Shell-LCC exhibits a distinct, controllable deblurring behavior across training iterations. Compared to the SFT baseline, which generates highly blurred motion during the handshake, our method (iter. 99 to 4999) gradually enhances local structure. However, continuous optimization triggers model collapse, with generated videos regressing towards the mean ($\hat{z}$). This phenomenon underscores a critical trade-off: low-level distortions, such as motion blur, are often entangled with high-frequency manifold information. 
Consequently, aggressively over-optimizing to penalize these blur artifacts forces the model to discard legitimate high-frequency information as well, ultimately breaking the generation manifold and driving the model into mean regression.

\begin{figure}[t]
    \centering

    \includegraphics[width=0.95\linewidth]{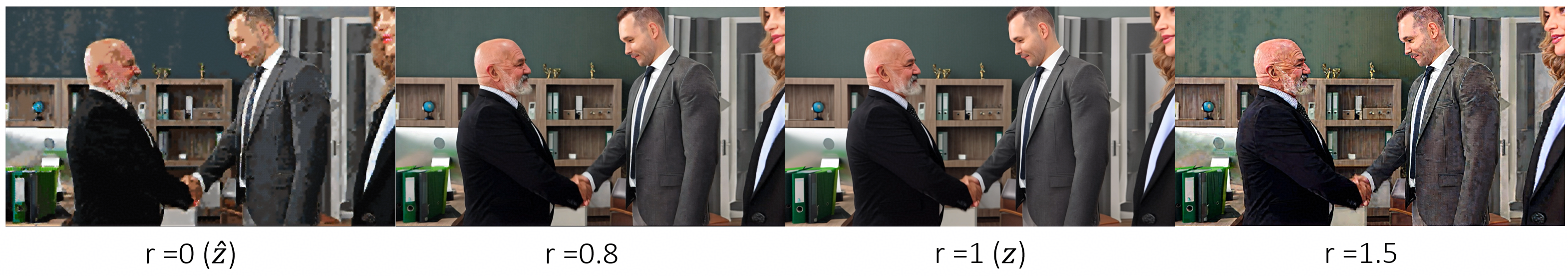}
    \caption{Radial reconstruction from $\hat{z}$. Videos transition from mean-like blur at $\hat{z}$ to sharper structures as the radius increases, while excessive deviation introduces distortions, revealing a shell-shaped latent manifold.}
    \label{fig:meanReg}
    \end{figure}

\textbf{Radial Reconstruction Visualization.}
To further investigate the geometric structure of the latent space, we perform a controlled radial reconstruction experiment. Starting from the LCC reconstruction $\hat z$ (the local convex combination), we sample latent codes along the outward radial direction: 
\[
z(r) = \hat z + r \cdot \frac{z - \hat z}{\|z - \hat z\|}.
\]

As shown in Fig.~\ref{fig:meanReg}, the decoded videos exhibit a clear three-stage transition. 
Reconstructions at $\hat z$ appear noticeably blurred and lack fine texture details, while moving outward along the radial direction gradually restores sharp structures and high-frequency components. 
However, when moving too far from $\hat z$, the videos become distorted again, although the artifacts differ from the mean-type blur observed at the center.
This radial behavior directly reveals the shell geometry of the latent manifold: the convex center corresponds to mean regression, while the realistic samples concentrate on a thin annular region surrounding it.

\begin{figure*}[t]
    \centering
    \includegraphics[width=\linewidth]{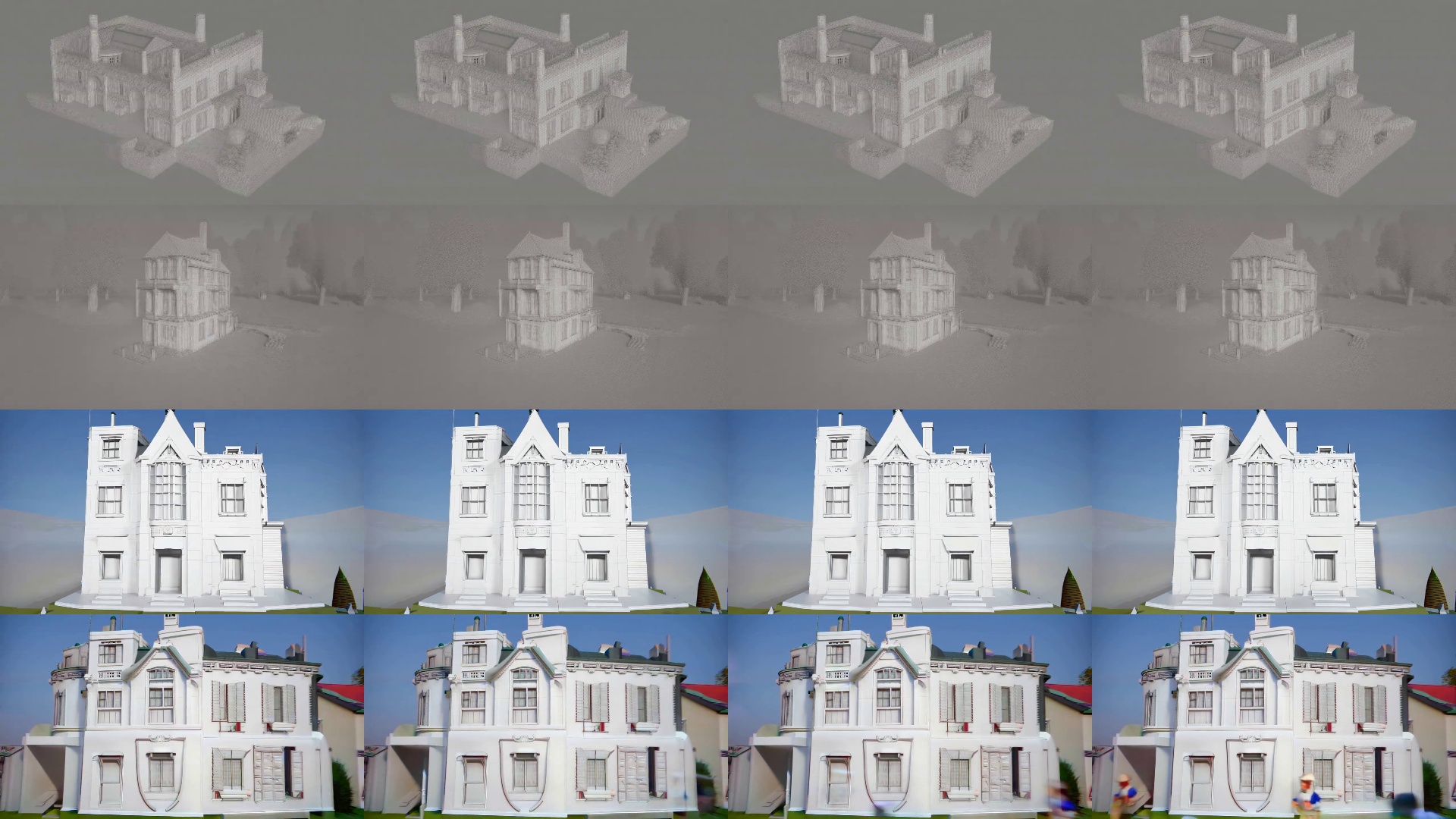}
    \caption{Qualitative comparison on open-source models, complementing the quantitative results in Tab.~\ref{tab:vbench}. From top to bottom: Wan-T2V-1.3B, Wan-T2V-1.3B~+~Shell-LCC, UltraWan-T2V-1.3B, and UltraWan-T2V-1.3B~+~Shell-LCC, for the prompt ``A 3D model of an 1800s Victorian house''. Adding Shell-LCC sharpens fine structures (e.g., window lattices and facade details) and suppresses the over-smoothing of the baselines, confirming that the geometric reward generalizes across open-source architectures.}
    \label{fig:open_source}
\end{figure*}

\section{Conclusion}
In this paper, we propose Shell-LCC, a novel, annotation-free geometric reward framework for text-to-video generation. By explicitly modeling the intrinsic SFT data manifold as an isotropic shell, we overcome the mean regression problem of standard approximations and establish a flexible point-to-surface alignment. Extensive experiments demonstrate that Shell-LCC preserves high-frequency structural details and significantly improves fine-grained imaging quality, achieving high-fidelity generation without trading off global aesthetics or relying on costly external reward models. The reward must nonetheless be optimized with care, since over-aggressive penalization can entangle genuine high-frequency details with low-level artifacts and re-induce mean regression (Fig.~\ref{fig:learning}), which we currently mitigate via early stopping. As our patch-level reward is orthogonal to preference-based methods such as DPO, combining the two and scaling Shell-LCC to larger backbones are promising directions for future work.

\bibliographystyle{splncs04}
\bibliography{main}

\end{document}